\documentclass[11pt,a4paper]{article}
\usepackage{times,latexsym}
\usepackage{url}
\usepackage[T1]{fontenc}

\usepackage[acceptedWithA]{tacl2018v2}
\usepackage[utf8]{inputenc}
\usepackage{xcolor}

\usepackage{amsmath}
\usepackage{amssymb}
\usepackage{amsthm}

\newcommand{\E}{\mathop{\mathbb{E}}}%

\usepackage{pslatex}
\usepackage[english]{babel}
\usepackage[utf8]{inputenc}
\usepackage{bm}
\usepackage{graphicx}
\usepackage{url}

\usepackage{natbib}

\newcounter{theorem}
\newtheorem{proposition}[theorem]{Proposition}

\title{Sensitivity as a  Complexity Measure for Sequence Classification Tasks}

\date{}
\author{Michael Hahn \\ Stanford University \\ \url{mhahn2@stanford.edu} \And Dan Jurafsky \\ Stanford University \\ \url{jurafsky@stanford.edu} \And Richard Futrell \\ University of California, Irvine \\ \url{rfutrell@uci.edu}}

\begin{document}

\maketitle

\begin{abstract}
We introduce a theoretical framework for understanding and predicting the complexity of sequence classification tasks, using a novel extension of the theory of Boolean function sensitivity.
The sensitivity of a function, given a distribution over input sequences, quantitfies the number of disjoint subsets of the input sequence that can each be individually changed to change the output.
We argue that standard sequence classification methods are biased towards learning low-sensitivity functions, so that tasks requiring high sensitivity are more difficult.
To that end, we show analytically that simple lexical classifiers can only express functions of bounded sensitivity, and we show empirically that low-sensitivity functions are easier to learn for LSTMs.
We then estimate sensitivity on 15 NLP tasks, finding that sensitivity is higher on challenging tasks collected in GLUE than on simple text classification tasks, and that sensitivity predicts the performance both of simple lexical classifiers and of vanilla BiLSTMs without pretrained contextualized embeddings.
Within a task, sensitivity predicts which inputs are hard for such simple models.
Our results suggest that the success of massively pretrained contextual representations stems in part because they provide representations from which information can be extracted by low-sensitivity decoders.
\end{abstract}

\section{Introduction}
What makes some tasks harder and others easier for modern machine learning methods?\footnote{Code is available at \url{https://github.com/m-hahn/sensitivity}.}
In NLP, simple models based on lexical classifiers provide good performance on some tasks, while strong performance on other tasks has been attained only recently with massive pretrained models.
However, there is no unified theoretical framework for understanding these difficulty differences between tasks, or what models might be more or less effective.

Existing complexity metrics provide limited practical insight. 
The Chomsky Hierarchy \citep{chomsky1956three} is a prominent classification of formal languages by complexity,
but it describes asymptotic worst-case complexity and does not provide a measure of how hard it is to achieve high accuracy on realistic task distributions.
Kolmogorov complexity \citep{li1993an} is uncomputable and becomes well-defined only in the asymptotic limit.
Psycholinguistic complexity metrics such as surprisal \citep{hale2001probabilistic} and dependency length \citep{gibson1998linguistic} only capture formal features of the input, without regard to the task.

We propose \textbf{sensitivity} as a theory of complexity for sequence classification tasks, i.e. any task involving learning a function from sequences to labels. The sensitivity of a function, given a distribution over input sequences, quantifies the number of disjoint subsets of the input sequence that can each be individually changed in such a way as to change the output.
Intuitively, high-sensitivity functions are complex because a single change in the input, in many different places, can completely change the output; low-sensitivity functions are simpler because the output is predictable from redundant information in many subsets of the input. We will argue that sensitivity predicts what tasks are easy or hard for modern machine learning methods to learn.

Our notion of sensitivity is grounded in a well-studied theory for Boolean functions \citep{odonnell2014analysis}, which we generalize to natural language.
Unlike measures like Kolmogorov complexity, sensitivity can be estimated on real datasets and single inputs without asymptotic approximations, only requiring a generalized language model such as XLNet~\citep{yang2019xlnet} and a strong model of the task.

In this paper, we argue that sensitivity captures informal notions of complexity both at the level of architectures and on the level of tasks.
First, we show that sensitivity quantifies architectural limitations and inductive biases of various machine learning architectures used in NLP, including both lexical classifiers and vanilla LSTMs without pretrained contextualized embeddings (Section~\ref{sec:bounds-models}).
Second, in a survey of 15 major NLP tasks, we find that sensitivity quantitatively predicts how difficult a task is for simple lexical classifiers and neural models, both across tasks and across different inputs for a single task (Section~\ref{sec:measuring}).
The validity of our methods for quantifying sensitivity is verified using human experiments in Section~\ref{sec:human}. Section~\ref{sec:discussion} discusses the relationship of sensitivity to previous theories of complexity and brittleness in neural networks, and implications for NLP practice. Section~\ref{sec:conclusion} concludes.

\section{Sensitivity}

\subsection{Analysis of Boolean Functions}
We build on notions of sensitivity developed for Boolean functions~\cite{kahn1988the,hatami2010variations,odonnell2014analysis}.
Analysis of Boolean functions is a powerful and rigorous theory with wide-ranging applications in theoretical computer science \citep{odonnell2014analysis}.
We first introduce the relevant notions, and then explain how these concepts can be generalized to the setting of fully general sequence classification.
The \textbf{sensitivity} of a Boolean function $f : \{-1,1\}^n \rightarrow \{-1,1\}$ at a bitstring $x \in \{-1,1\}^n$ is defined as:
\begin{equation}\label{eq:sens-def-boolean}
    s(f,x) = \sum_{i=1}^n 1_{f(x)\neq f(x^{\oplus i})},
\end{equation}
where $x^{\oplus i}$ is the result of flipping the $i$-th bit of $x$.
This describes how many bits of $x$ can be flipped individually to change $f$, or equivalently, how many Hamming neighbors of $x$ have the opposite value of $f$.

The highest possible sensitivity is attained by the  \textsc{Parity} function $f_\text{Parity}(x) := \prod_{i=1}^n x_i$.
Given a string of ``1''s and ``-1''s, this function counts whether the number of negative inputs is even (output $+1$) or odd (output $-1$). For instance, $f_\text{Parity}(1, 1, 1) = f_\text{Parity}(1, -1, -1) = 1$ and $f_\text{Parity}(1, -1, 1) = f_\text{Parity}(-1, 1, 1) = -1$.
The function $f_\text{Parity}$ has the property that flipping any individual bit flips the output.
For instance, given the string ``$1\ 1\ 1$'', changing any of the three input symbols to ``$-1$'' flips the parity of the string from $+1$ to $-1$.
Therefore, for every bitstring $x \in \{-1,1\}^n$, we have $s(f_{Parity}, x) = n$.
It is impossible to approximate $f_\text{Parity}$ beyond chance level with linear functions \citep{minsky1969perceptrons}, or with linear combinations of functions that contain nonlinear interactions between less than $n$ input bits \citep{odonnell2014analysis}.
In this sense, the function $f_\text{Parity}$ is maximally nonlinear.
On the other hand, low-sensitivity functions can be approximated with linear functions or linear combinations of functions that each only combine a few input bits \citep[Thm. 2.38]{odonnell2014analysis}.
Sensitivity also has close connections with other complexity measures such as decision tree depth \citep{nisan1991crew} and the degree of a Boolean function when written as a polynomial.

\subsection{Application to sequence classification}
We argue that this theory can be brought to bear to quantify the complexity of sequence classification tasks.
In this setting, sensitivity measures the nonlinearity of the decision boundary.
Low sensitivity tasks are those where simple methods based on linear combinations of local features are most successful.
For instance, low sensitivity tasks can be solved by bag-of-words classifiers and linear classifiers based on $n$-gram features, which have bounded similarity (as we will make precise in Proposition~\ref{prop:sensitivity-linear-model} below). 
On the other hand, high sensitivity tasks require more sophisticated methods.
We expect that tasks which have proven empirically difficult in the literature, such as those requiring reasoning, correspond to those with high sensitivity, which means that changing different substrings in an input can easily flip the label (e.g. \textsc{Entailment} $\Rightarrow$ \textsc{NonEntailment}).

Testing these ideas requires generalizing sensitivity to functions more akin to those relevant in NLP along several aspects.
One aspect can be dealt with without major changes: NLP tasks are defined on alphabets $\Sigma$ with more than two elements, such as the words of a language.
The theory can be accommodated to such alphabets, leading to a generalized definition of sensitivity applicable when the symbols $X_i$ are distributed independently and uniformly \citep[rephrased based on][Def. 8.22]{odonnell2014analysis}:
\begin{equation}\label{eq:sensitivity-variance-generalized}
	s(f,x) := \sum_{i=1}^n \operatorname{Var}\left(f(X)| \forall j\neq i : X_j = x_{j} \right),
\end{equation}
where the variance measures how much $f$ varies across strings $X \in \Sigma^n$ that agree with $x$ on all except possibly the $i$-th input.
Definition (\ref{eq:sensitivity-variance-generalized}) reduces to (\ref{eq:sens-def-boolean}) if $\Sigma = \{-1,1\}$ and $f : \{-1,1\}^n\rightarrow \{-1,1\}$. 

More challenging is the fact that symbol sequences in language are not distributed uniformly.
For example, in movie review sentiment classification, most inputs will sound like movie reviews (rather than tweets or Wikipedia articles), and almost all will respect the grammatical and statistical properties of the underlying language. 
When defining a generalization of $s(f,x)$ to natural language, we want to focus on those strings $x$ and their Hamming neighbors $x'$ that are typical instances of the problem.
We next describe an adaptation of (\ref{eq:sens-def-boolean}-\ref{eq:sensitivity-variance-generalized}) taking this into account.

\subsection{Formal Definitions}

In order to adapt the idea of sensitivity to the setting of NLP tasks, we introduce a generalized notion called block sensitivity. Block sensitivity is the maximum sensitivity over all possible partitions of the input bits.
Block sensitivity has been studied for Boolean functions as an upper bound on~(\ref{eq:sens-def-boolean}) \citep{nisan1991crew,bernasconi1996sensitivity,hatami2010variations}; we construct a probabilistic version of this notion as a sensitivity measure appropriate to more sequence classification tasks.

Consider a set $\Sigma$ (e.g., the words of a language), with an arbitrary distribution $\Pi$ over the set $\Sigma^*$ of finite sequences of symbols from $\Sigma$.
We formalize classification tasks as functions $f : \Sigma^* \rightarrow [-1,1]$.\footnote{For multi-class problems, we take a family of functions $f$ corresponding to the classes, see Section~\ref{sec:measuring}.} 
Such functions could be binary classifiers $f$ mapping to $\{-1, 1\}$, or they could output a continuous score.
We take the output space to be $[-1,1]$ instead of $[0,1]$ to make our definitions consistent with those from the analysis of Boolean functions. 

The \textbf{subset sensitivity} of the function $f : \Sigma^* \rightarrow \mathbb{R}$ on the point $x \in \Sigma^n$ and the set $P \subset \{1, \dots, n\}$ is
\begin{equation}\label{eq:sensitivity-variance}
	s(f,x,P) := \operatorname{Var}\left(f(X)|X \in x^{\oplus P}\right),
\end{equation}
where $x^{\oplus P}$ denotes the set of all strings $x'$ that agree with $x$ on all indices outside of $P$: 
\begin{equation}
	x^{\oplus P} := \{x' \in \Sigma^{n} : x'_j = x_j \text{ for all }j \in \{1,\dots,n\} - P \},
\end{equation}
and the variance is computed with respect to $\Pi$.
If $P$ is a singleton $\{i\}$, we recover the term inside the sum in (\ref{eq:sensitivity-variance-generalized}): $s(f,x) = \sum_{i=1}^n s(f,x,\{i\})$.

We illustrate this definition in Figure~\ref{fig:sensitivity-single-sentiment} with examples from the Stanford Sentiment Treebank~\citep{socher2013recursive}.
Here, the function $f$ maps movie reviews to the probability that the review is positive, scaled to $[-1, 1]$. 
For each sentence, we select a singleton subset $P$ and show 10 samples from $\Pi$, the distribution over possible substitutions.
In Sentence 1, due to the positive adjectives in the context, the distribution is concentrated on positive adjectives, and so the sensitivity $s(f,x,P) \approx 0$.
In Sentence 2, both positive and negative adjectives are plausible substitutions, and $s(f,x,P) \approx 0.6$.

This example shows how~(\ref{eq:sensitivity-variance}) differs from the vanilla definition~(\ref{eq:sens-def-boolean}) by accounting for the statistical dependencies between words in natural language:
it takes into account that the choice of possible completions for a set $P$ is often constrained by the context given by $x$.
Inputs violating these statistical dependencies (e.g., `a \emph{boring}, witty, seductive movie' for Figure~\ref{fig:sensitivity-single-sentiment}) are unlikely to occur in naturalistic input, and the behavior of $f$ on such unlikely inputs may not impact the difficulty of representing $f$ with high average fidelity.
This motivates considering the variance of $f$ over neighboring strings, instead of, say, the entire range of $f$ over all possible neighboring strings.

Based on subset sensitivity, we introduce the \textbf{block sensitivity} at $x$ as an analogue to (\ref{eq:sens-def-boolean}):
\begin{equation}\label{eq:block-sens}
    bs(f,x) := \max_{k, P_1 \dot\cup \dots \dot\cup P_k} \sum_{i=1}^k s(f,x,P_i),
\end{equation}
where the maximization ranges over all partitionings of $\{1, \dots, n\}$ into disjoint subsets $P_1 \dot\cup \dots \dot\cup P_k$ ($\dot\cup$ denoting disjoint union).
We recover the quantity $s(f,x)$ (\ref{eq:sens-def-boolean}-\ref{eq:sensitivity-variance-generalized}) by restricting subsets $P_i$ to the singletons $\{i\}$; thus, we have
\begin{equation}
    bs(f,x) \ge s(f,x).
\end{equation}
Intuitively, $bs(f,x)$ measures the following:
Given an input $x$, how many disjoint subsequences can be changed individually so as to flip the label?
The formal definition modifies this logic by considering, for each subsequence, not whether changing it to flip the label is possible in principle, but also the probabilities of the different changes.
A useful summary statistic is the \textbf{average block sensitivity}:
\begin{equation}
\widehat{bs}(f) = \E_{x \sim \Pi} \left[ bs(f,x) \right].
\end{equation}

\begin{figure*}
\textbf{Sentence 1:}

\begin{tabular}{lllllll} 
a  \textcolor{blue}{gorgeous} , witty , seductive  movie . ($+1$) \\
\small{seductive ($+1$), brilliant ($+1$), cute ($+1$), sexy ($2\times$, $+1$), shocking ($2\times$, $+1$), stylish ($+1$), charming ($+1$)}
\end{tabular}
        
\textbf{Sentence 2:}

\begin{tabular}{lllllll} 
a painfully \textcolor{blue}{funny}  ode to  bad behavior ($+1$)\\
	\small{bleak ($-1$), small ($-1$), accurate ($2\times$, $-0.2$), sexy (0.13), true (0.5), beautiful (0.9) funny ($2\times$, $+1$), honest ($+1$) }
\end{tabular}
        
  \caption{Subset sensitivity (\ref{eq:sensitivity-variance}) for sentiment analysis, for two inputs from the SST-2 dev set. For each inputs, we select a one-word subsequence (marked in blue, corresponding to sets $\{2\}$ for Sentence 1, and $\{3\}$ for Sentence 2), and show 10 possible substitutions sampled using XLNet  (see Section~\ref{sec:measuring}; ``$2\times$'' indicates samples appearing twice).
  We show sentiment prediction (between $-1.0$ for negative and $+1.0$ for positive sentiment), obtained using  RoBERTa  (see Section~\ref{sec:measuring}), both for the original sentence and each version arising from substituting any of the other adjectives.
  In Sentence 1, due to the presence of positive adjectives in the context, the distribution is concentrated on positive adjectives; $f(x') = +1$ for each sampled $x' \in x^{\oplus P}$. Therefore, subset sensitivity $s(f,x,P)$ is estimated as $0.0$.
In Sentence 2, both positive and negative adjectives are plausible substitutions, and $s(f,x,P) = 0.58$.}\label{fig:sensitivity-single-sentiment}
\end{figure*}

\paragraph{Why consider subsets?}
By considering subsets $P$ instead of single indices $i$, block sensitivity takes into account that words are composed into phrases, and that changing a phrase might change the meaning when changing any individual word cannot.
For instance, exchanging the entire phrase `a gorgeous, witty, seductive' (see Figure~\ref{fig:sensitivity-single-sentiment}) with something negative can make the review negative, whereas exchanging any of the individual adjectives cannot, due to the statistical dependencies between the different words. 
This definition also makes the sensitivity measure robust against tokenization: a more fine-grained tokenization (e.g., into characters) cannot decrease $bs(f,x)$.

\section{Sensitivity Bounds for NLP Methods}\label{sec:bounds-models}

Many statistical NLP methods proposed over the past decades involve linear combinations of features that look at individual words or groups of a few words.
Proposition~\ref{prop:sensitivity-linear-model} shows that such methods can only express functions of bounded block sensitivity, with an upper bound quadratic in the number $k$ of inputs the model looks at simultaneously, independently of input length $n$.

\begin{proposition}\label{prop:sensitivity-linear-model}
Let $f$ be any function $\Sigma^* \rightarrow \mathbb{R}$ parameterized as follows:
	\begin{equation}\label{eq:simple-model}
f(x) := h\left(\frac{1}{n} \sum_{i=1}^{n-k} f_{i,n}(x_i, \dots, x_{i+k})\right),\ \ \ \ (x \in \Sigma^n)
\end{equation}
	where $f_{1,n}, \dots, f_{{n-k},n}$ are functions $\Sigma^k \rightarrow \mathbb{R}^d$ such that $\max_{x \in \Sigma^k} \|f_{i,n}(x)\|_2 \leq C$, and $h : \mathbb{R}^d \rightarrow \mathbb{R}$ is $L$-Lipschitz continuous.
Then, independently of input length $n$, we have
\begin{equation}
    bs(f,x) \leq 2L^2 C^2 k^2.
\end{equation}

\end{proposition}

\begin{proof}
Fix a partition $P_1 \dot\cup \dots \dot\cup P_l = \{1, \dots, n\}$.
Write $g(x)$ for the average inside $h(\cdot)$ in (\ref{eq:simple-model}).
	Changing inputs in $P_i$ affects up to $k |P_i|$ of the summands in $g$.
	The $\ell_2$ norm of the sum of these affected terms is bounded by $\frac{Ck |P_i|}{n}$, and thus
\begin{align*}
 Var(f|X \in x^{\oplus P_i}) &= \frac{1}{2} \E_{X, Y \in x^{\oplus P_i}} |f(X)-f(Y)|^2 \\
	&\leq \frac{1}{2} L^2 \E_{X, Y \in x^{\oplus P_i}} \|g(X)-g(Y)\|^2_2 \\
	& \leq 2 \frac{L^2C^2k^2 |P_i|^2}{n^2}.
\end{align*}
Given $\sum_{i=1}^k |P_i|^2 \leq (\sum_{i=1}^k |P_i|)^2 = n^2$, we find
	\begin{equation*}
	\sum_{i=1}^l s(f,x,P_i)   \leq 2 \frac{L^2C^2k^2}{n^2} \sum_{i=1}^k |P_i|^2 \leq 2L^2 C^2k^2.
\end{equation*}
\end{proof}

This result has direct bearing on a wide variety of methods used in NLP, such as averaging word embeddings to construct sentence embeddings~\cite{wieting2015towards, arora2016simple, ethayarajh2018unsupervised}, CNNs~\citep{kim2014convolutional} with average pooling,
 and log-linear models with $n$-gram features.
The parameter $k$ equals $1$ for models averaging word embeddings, the Kernel width for CNNs with average pooling, and $n$ for models using $n$-gram features.
$C$ describes the norm of word embeddings, of the output of a CNN kernel, or of the weights of a linear model.
Lipschitz functions $h$ include the sigmoid function $\sigma$ used in logistic regression and its generalization \textit{softmax}, which are $1$-Lipschitz, and feedforward networks with Lipschitz activations.

RNNs and LSTMs \citep{hochreiter1997long} can express functions of any sensitivity, such as $f_{\text{Parity}}$, because they can express all regular languages~\citep{horne1994bounds}.
On the other hand, transformers~\citep{vaswani2017attention} have asymptotically bounded sensitivity as the input length $n$ increases~\citep[Lemma 5]{hahn2020theoretical}. 

\begin{figure*}
    \centering
\includegraphics[width=0.45\textwidth]{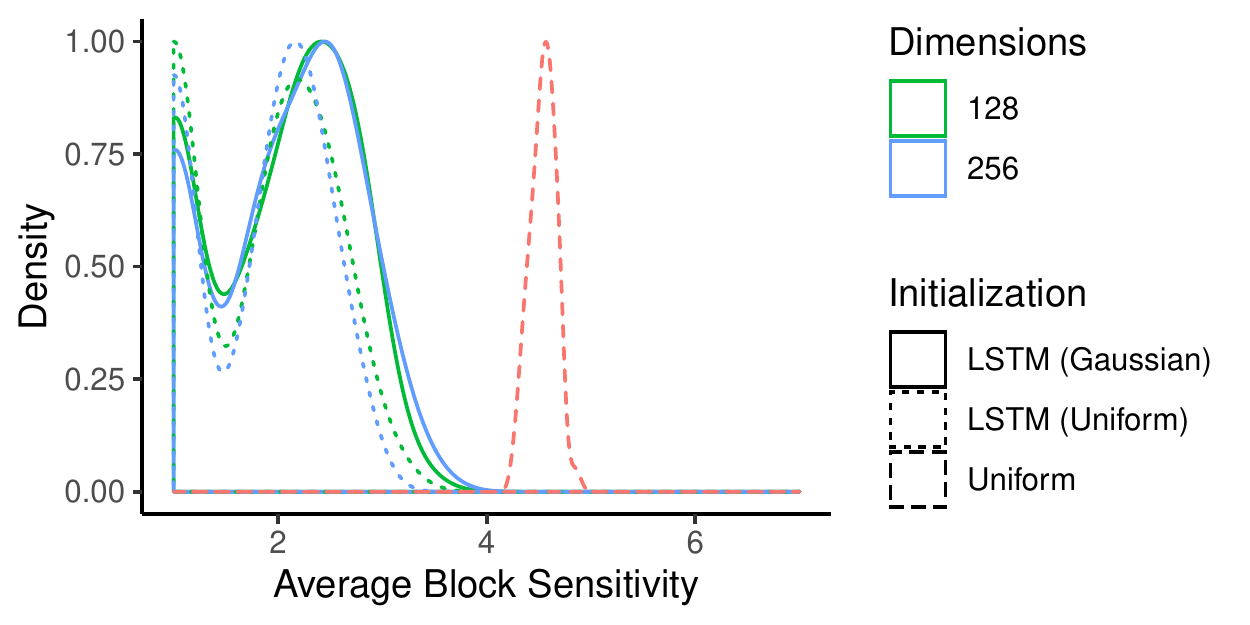}
\includegraphics[width=0.5\textwidth]{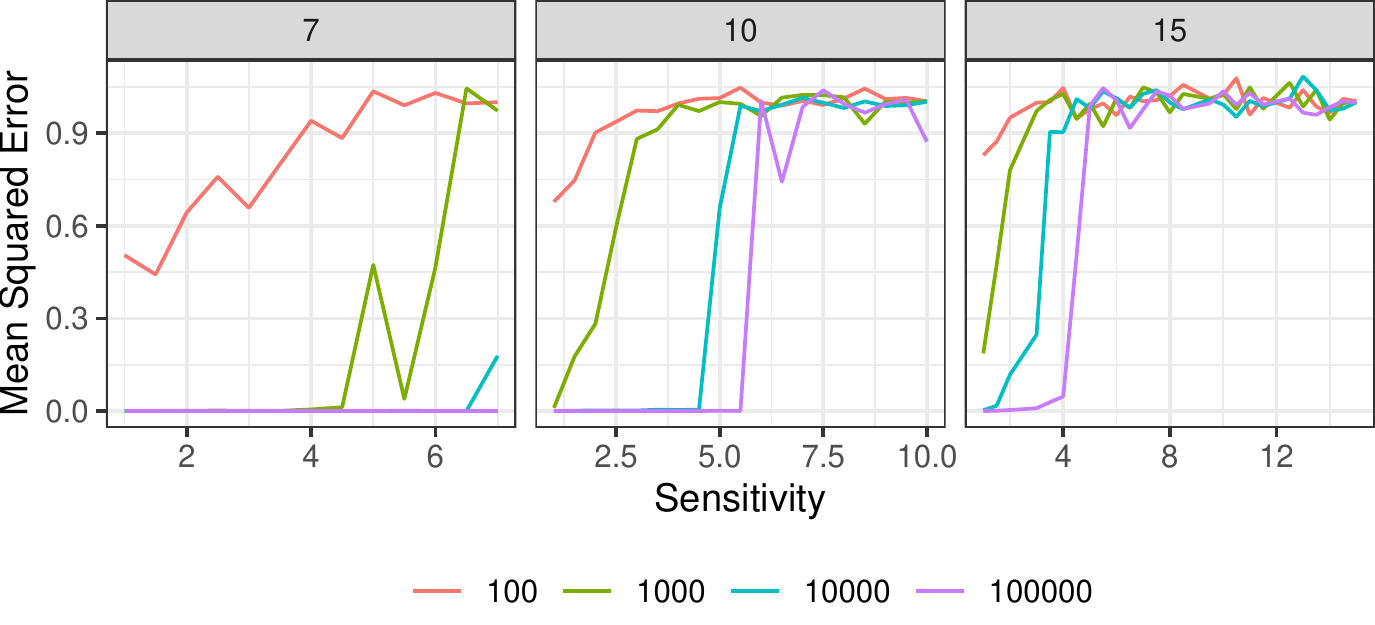}
	\caption{LSTMs are biased towards low-sensitivity functions. (1) Left: Distribution of sensitivity of Boolean functions defined by randomly initialized LSTMs (green and blue) and by the uniform distribution (red) over functions $f : \{-1,1\}^7 \rightarrow \{-1,1\}$. (2) Right: Losses for an LSTM (128 hidden units) fitting random functions $f : \{-1,1\}^{N} \rightarrow \mathbb{R}$ ($N=7, 10, 15$) with given sensitivities, after $10^2$, $10^3$, $10^4$, $10^5$ iterations of training.}
    \label{fig:learnability}
\end{figure*}

We show that even LSTMs have a learning bias towards low-sensitivity functions, despite their theoretical capacity to represent high-sensitivity functions.
We consider functions  $f : \{-1,1\}^n \rightarrow \mathbb{R}$ where inputs are uniformly distributed over $\{-1,1\}^n$. 
We first evaluated average block sensitivity both for randomly initialized LSTMs and for the uniform distribution over Boolean functions $\{-1,1\}^7 \rightarrow \{-1,1\}$. 
We constructed Boolean functions from a randomly initialized LSTM by obtaining a scalar output and making this a binary output $f$ based on a threshold chosen to maximize $\operatorname{Var}(f)$.
We initialized the LSTM's weights uniformly from $[-d^{-0.5}, d^{-0.5}]$ or from a Gaussian with $\sigma^2 = d^{-1}$, where $d$ is the number of hidden units.
Results are shown in Figure~\ref{fig:learnability}, for $d=128$ and $d=256$.
Random Boolean functions have block sensitivity tightly concentrated around $\approx 4.5$, whereas the randomly initialized LSTMs consistently show lower block sensitivity.
This suggests that low-sensitivity functions are `overrepresented' in the LSTM parameter space, echoing a theoretical result for feedforward networks~\citep{de2018deep}.

Second, we directly examined learnability on functions of different sensitivities.
As randomly chosen functions have tightly clustered sensitivity, we sampled\footnote{For each $i = 1, \dots, N$, we sampled functions $f$ where the Fourier spectrum is entirely concentrated on degrees $\{i-1, i, i+1\}$. By \citet[Thm. 2.38]{odonnell2014analysis}, $as(f) \approx i$.} functions with a specific targeted average sensitivity $as(f) = \frac{1}{2^{n}} \sum_{x\in\{-1,1\}^n} s(f,x)$.
We did this for sequence lengths $n=7,10,15$.
For each $i=1, \dots, n$, we constructed five such functions, and then trained an LSTM (128 hidden units) for $10^5$ iterations with Adam (learning rate $0.003$, batch size $32$), and recorded average mean squared error after $10^2$, $10^3$, $10^4$, $10^5$ training iterations.
Training batches and test examples are sampled uniformly from the $2^n$ elements of $\{-1,1\}^n$, without consideration of out-of-sample generalization.
Results are shown in Figure~\ref{fig:learnability}.
For $n=7$, we arrange functions by $\widehat{bs}(f)$; for $n=10, 15$ we take $as(f)$ instead as it can be computed efficiently and is strongly correlated with $\widehat{bs}(f)$ at $n=7$ ($R = 0.95$).
Low-sensitivity functions are learned perfectly with fewer iterations, whereas high-sensitivity functions are not approximated much better than chance even after $10^5$ training iterations.
We note that this is a result on the ability to simply \emph{fit} a function of $2^n$ inputs, not the (harder) task of \emph{generalizing} to unseen input.

\section{Sensitivity and Difficulty of NLP Tasks}\label{sec:measuring}

In Section~\ref{sec:bounds-models}, we provided evidence that sensitivity describes how hard a function is to learn and represent for simple machine learning architectures that do not include pretrained contextual embeddings.
In this section, we argue empirically that sensitivity is successful at capturing intuitive notions of task difficulty: Low-sensitivity tasks are those on which simple classifiers as described in Proposition~\ref{prop:sensitivity-linear-model}, and vanilla LSTMs without pretraining, are relatively successful.
More challenging tasks such as those collected in the GLUE suite~\citep{wang2019glue} have higher sensitivity.

Estimating block sensitivity~(\ref{eq:block-sens}) requires two ingredients:
an estimate of the distributions of neighboring strings $\Pi(X|X \in x^{\oplus P})$, and an estimate of $f$ on this set.
We approximate $\Pi$ via a language model, and $f$ via a trained model that is known to attain strong performance on the task. 
That is, we estimate the sensitivity of a task $f$ by measuring the sensitivity of a model $f'$ that is known to provide close fit to $f$ on the task's input distribution.
In Section~\ref{sec:human}, we report human annotation studies that justify this approximation.

\paragraph{Sampling Neighboring Strings}
For estimating $\Pi(X|X \in x^{\oplus P})$, we leverage the ability of XLNet~\citep{yang2019xlnet} and u-PMLM~\citep{liao2020probabilistically} to model prediction in any order.
We use the pretrained \texttt{xlnet-large-cased} model provided in~\citet{wolf2019huggingface}, and a pretrained u-PMLM model trained on the 1 Billion Word benchmark~\citep{chelba2014one}.
As these models take input on the level of subword tokenizations, we require all samples to consist of the same number of subword symbols as in the span covered by $P$.
To enable meaningful comparison with traditional tokenization and with human intuitions, we only consider subsets $P$ that respect whitespace.
We take 10 samples for each $P$.
For tasks with short inputs (text classification and CoLA), we finetune XLNet on the training set to produce completions more in line with the task-specific input distribution.
Due to compute availability, we did not apply this procedure to other tasks.
Finetuning XLNet slightly increased estimated sensitivity; as we applied it to those tasks expected to have low sensitivity, this procedure potentially makes comparison between tasks more conservative.

\paragraph{Tasks}
First, we consider four text classification tasks: movie review sentiment \citep[MR,][]{pang2005seeing}, sentence subjectivity \citep[SUBJ,][]{pang2004a}, customer reviews sentiment \citep[CR,][]{hu2004mining}, 
and opinion polarity \citep[MPQA,][]{wiebe2005annotating}. 
On these tasks, low-sensitivity models such as CNNs are known to achieve good performance~\citep{kim2014convolutional}.
To approximate the functions $f$, we finetune \texttt{roberta.large.mnli} using \texttt{fairseq} for each of the tasks using a single set of hyperparameters.

Second, we selected all tasks of the GLUE challenge suite~\citep{wang2019glue}, designed to require a good amount of nontrivial language understading.
GLUE contains inference, similarity, and paraphrase tasks (MNLI, \citet{williams2018a}; MRPC, \citet{dolan2005automatically}; QNLI, \citet{rajpurkar2016squad}; QQP; STS-B, \citet{cer2017semeval}; RTE, \citet{dagan2009recognizing}), an NLI version of the Winograd schema challenge \citep{levesque2012the}, linguistic acceptability judgments \citep[CoLA,][]{warstadt2019neural}, and the Stanford sentiment treebank \citep[SST-2,][]{socher2013recursive}.
On many of these tasks, simple BOW baselines perform essentially at chance~\citep{wang2019glue}.
We obtain predictions by finetuning RoBERTa (\texttt{roberta.large.mnli}) using \texttt{fairseq} \citep{ott2019fairseq} using provided hyperparameters\footnote{\url{https://github.com/pytorch/fairseq/blob/master/examples/roberta/README.glue.md}, retrieved June 1, 2020.}.
RoBERTa provides performance close to or exceeding estimated human performance on all GLUE tasks. 
For the Winograd schema challenge, we took the WSC version from SuperGLUE~\citep{wang2019superglue} instead of the NLI reformulation (WNLI) used in GLUE; we used the pretrained model \texttt{roberta.large.wsc}. 
Unlike WNLI, WSC is a single-span task, reducing the number of subsets $P$ considered. 

Third, we considered sequence classification formulations of POS tagging and syntactic parsing.
For 150 dev sentences in the English Web Treebank~\citep{silveira2014a}, we considered the word at the median position of the sentence, and estimated sensitivity of identifying (1) its POS tag in the universal tagset~\citep{petrov2012a}, (2) its Universal Dependencies label~\citep{nivre2016universal}, (3) the relative position of its head, as an integer.
All three tasks are formalized as multi-class classification problems.
We estimated all three computations using the pretrained English dependency parser provided in Stanza~\citep{qi2018universal,qi2020stanza}.

Fourth, we considered two datasets probing syntactic knowledge of anaphor licensing~\citep{marvin2018targeted,hu2020a}, namely tasks 248 and 260 in SyntaxGym \citep{Gauthier:et-al:2020:syntaxgym}.
These tasks ask a model to choose a singular (\emph{himself}) or plural (\emph{themselves}) reflexive after a context where only one is grammatical, but identifying the right reflexive requires syntactic knowledge.
We modeled $f$ using the medium-sized GPT2 model \citep{radford2019language}.
We chose this task because it could be formalized as binary classification problem, and because GPT2 performed better on this task than on the more familiar subject-verb agreement (and on the feminine version with \emph{herself}).

For each task, we estimated sensitivity for at least 150 dev examples, determined by compute availability.
For the syntactic tasks, we estimated sensitivity on the full dataset, as language models are evaluated on these tasks without finetuning.

We considered continuous predictions in $[-1, 1]$ for binary classification tasks, and in $[-1,1]^d$ for multiclass tasks with $d$ classes, obtained from the sigmoid or softmax layer of the relevant models.
For STS-B, we rescale continuous similarity scores to $[-1, 1]$.
For parsing and WSC, we used the discrete output labels provided by the pretrained models, represented as one-hot vectors $\in \{-1,1\}^d$ or binary labels $\in \{-1,1\}$.
For multivariate output $f(x) \in [-1,1]^d$, we define $s(f,x,P)$ by computing it for each of the $d$ coordinates of $f(x)$, and taking the maximum value over these.
The resulting sensitivity estimates describe the behavior of the coordinate of $f$ that has the most nonlinear decision boundary around $x$.

\paragraph{Lower bound approximation}
Calculating block sensitivity (\ref{eq:block-sens}) requires calculating the variance for each of the exponentially many subparts $P$ of the input, intractable for all but short inputs.
We restrict consideration to a polynomial number of subparts, thus obtaining a \emph{lower bound} on full block sensitivity.
We only consider (1) subsets of $1, \dots, 8$ adjacent tokens, and  (2) unions of sets $\{x_{in/7}, \dots, x_{(i+1)n/7-1}\}$ for $i=1, \dots, 7$.
For the parsing tasks, we additionally consider all subsets in a window of 7 tokens around the relevant word.
This bounds the number of subsets by $8n+256$, compared to $2^n$ for full block sensitivity.

\begin{figure}
    \centering
	\footnotesize{GLUE}

    \includegraphics[width=0.45\textwidth]{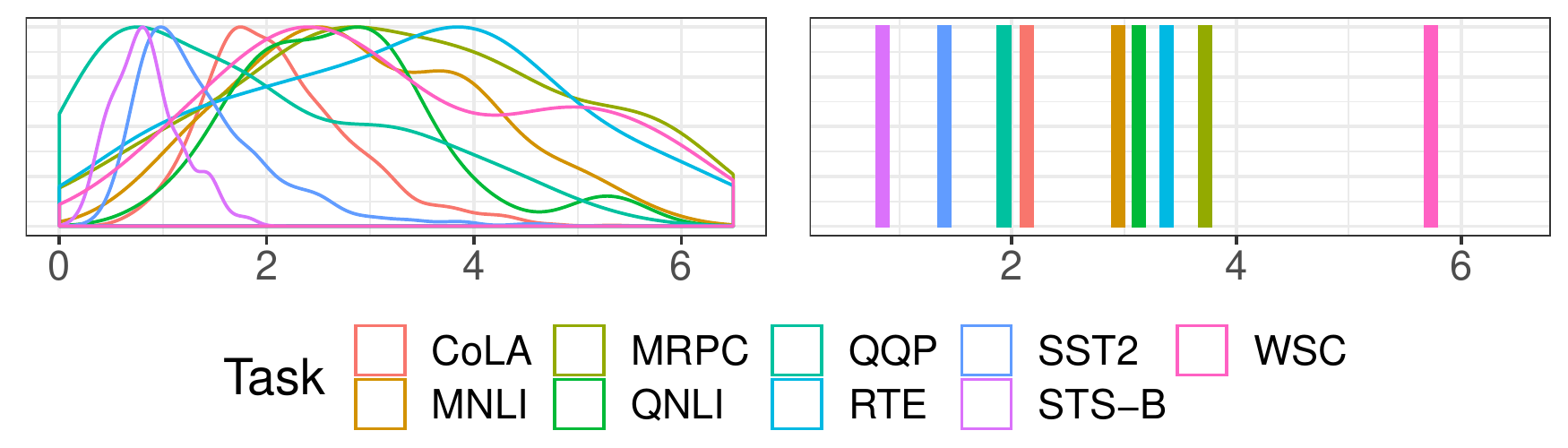}

	Parsing

    \includegraphics[width=0.45\textwidth]{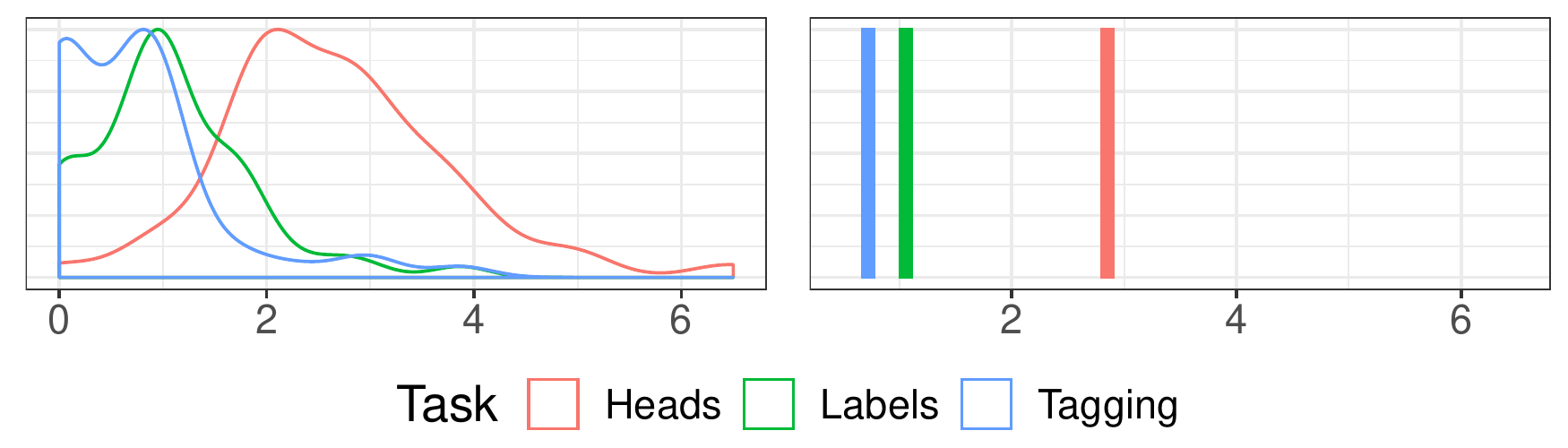}

	Syntax

    \includegraphics[width=0.45\textwidth]{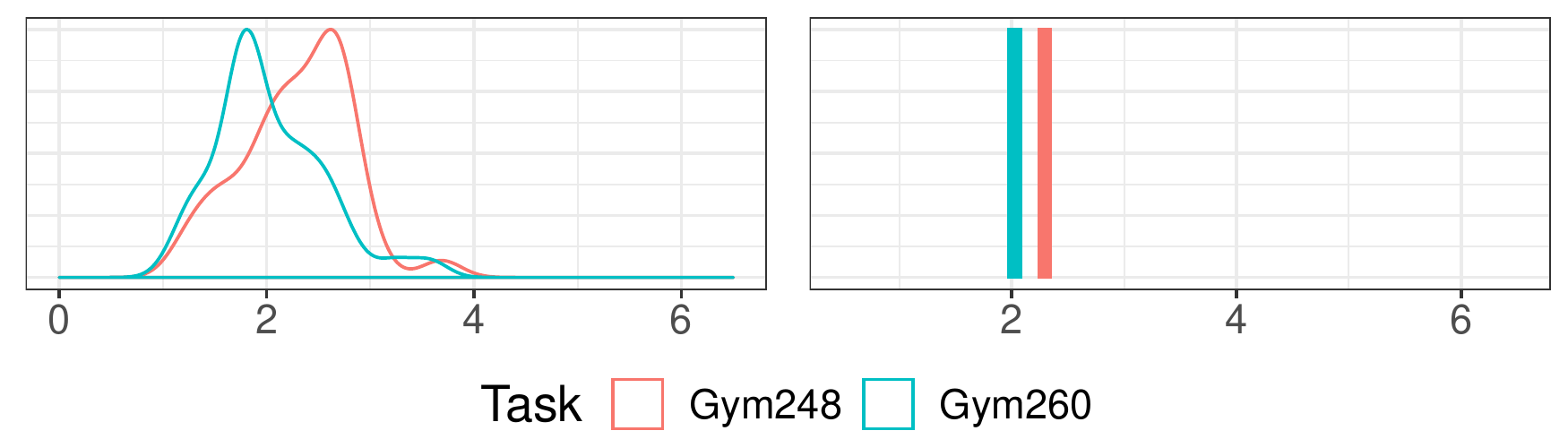}

	Text Classification
    \includegraphics[width=0.45\textwidth]{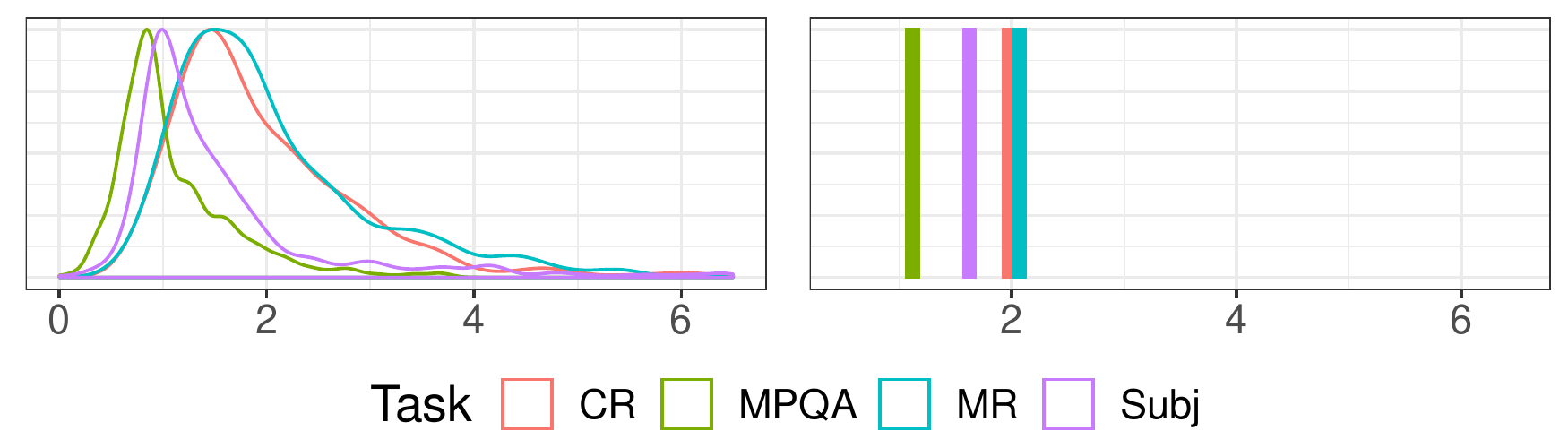}
\caption{Block sensitivity: For each task, we provide a smoothed histogram of the block sensitivity per input (left), and average block sensitivity (right). Estimates obtained using XLNet; compare Figure~\ref{fig:sensitivity-cbow-acc} for u-PMLM.}
    \label{fig:histograms-byTask}
\end{figure}

\subsection{Results}

Across the 15 tasks, XLNet and u-PMLM yielded very similar estimates of average block sensitivity ($R=0.87$, $p=7 \cdot 10^{-6}$). 
In Figure~\ref{fig:histograms-byTask}, we show block sensitivity across tasks as estimated by XLNet.
The left panels show kernel density estimates of the distribution over $bs(f,x)$ over the inputs $x$ from the dev sets.
The right panels show estimated average block sensitivity $\widehat{bs}(f)$.
Text classification tasks have low estimated block sensitivity, with $bs(f,x)$ being concentrated on values lower than three.
For the two syntactic tasks, sensitivity is slightly higher; in comparison to the text classification tasks, the histograms show that these tasks have no datapoints with very low sensitivity.
For parsing, we see a substantial difference between POS tagging and relation labeling on the one hand, and head identification on the other hand.
Identifying tags and relations has lower sensitivity comparable to text classification tasks, whereas identifying the relative position of the head has higher sensitivity.
This makes sense: The relative position of the head is sensitive to intervening words that, while not changing the syntactic relation, change the numerical distance between head and dependent.
Finally, for GLUE, we observe a wide range of sensitivity scores.
SST-2, a sentiment analysis task, has sensitivity very similar to the (other) text classification tasks, as do STS-B (semantic similarity) and QQP (identifying redundant Quora questions).
Other tasks show substantially higher scores; the highest estimated average block sensitivities are attained by RTE, MRPC, and WSC, three tasks designed to require nontrivial reasoning.

To provide insight into these results, we show examples from SST-2 and RTE, with samples from XLNet.
In Figure~\ref{fig:ex-sst2}, we show two examples from SST-2. 
The first example has low sensitivity, as our models find only one sensitive subset.
On the second example, our models find three disjoint sensitive subsets, leading to higher sensitivity.
In Figure~\ref{fig:ex-rte}, we show an example from RTE, consisting of a premise and a hypothesis.
The models identify five highly sensitivity subsequences, such that changing the input on any of these subsequences can flip the label from \textsc{Entailment} to \textsc{NoEntailment}.

\paragraph{Sensitivity and Sentence Length}
Sensitivity might be higher on longer sentences, because they can be partitioned into more sets $P$.
Does this explain away the differences between tasks?
Figure~\ref{fig:sens-by-length} shows per-sentence sensitivity (estimated using XLNet) as a function of sentence length.
The left panel compares sensitivity on simple text classification tasks and on CoLA, a GLUE task consisting of short sentences.
For the simple text classification tasks, sensitivity increases sharply for very short sentences, but then plateaus.
For CoLA, it increases with length.
The right panel shows averaged values $bs(f,x)$ across the tasks in each of the four categories.
Again, sensitivity increases for GLUE and dependency parsing, while it plateaus for text classification.
The two syntactic tasks consist of short and tightly controlled sentences; in relation to their lengths, their sensitivities are particularly high. 

\begin{figure}
        \textbf{Low Sensitivity:}
        
\setlength{\tabcolsep}{0.15em}
            \begin{tabular}{lllllll} 
      &  &    \textcolor{blue}{a gorgeous , witty , seductive} & movie . \\
         1.& & \textcolor{red}{a farce of ideas squanders this} & movie . \\
        \end{tabular}
        
        \textbf{High Sensitivity:}
        
\setlength{\tabcolsep}{0.15em}
        \begin{tabular}{lllllll} %
        & \textcolor{blue}{a painfully} & \textcolor{blue}{funny} & \textcolor{blue}{ode to} & bad & \textcolor{blue}{behavior} . \\
        1. & \textcolor{red}{Not a} & funny & \textcolor{red}{story, just} & bad & behavior . \\
        2. & a painfully & \textcolor{red}{bleak} & ode to & bad & behavior . \\
        3. & a painfully &funny& ode to & bad & \textcolor{red}{movies} . \\
        \end{tabular}

    \caption{Two inputs from SST-2.
	The first one has low block sensitivity (0.93), as our models find only one sensitive subset $P$.
	We show one completion sampled from $x^{\oplus P}$ that flips the label predicted by RoBERTa from \textsc{Positive} to \textsc{Negative}.
	The second input has higher block sensitivity (1.88), with three disjoint sensitive subsets.
	For each subset, we show a completion sampled using XLNet that flips the predicted label. 
    }
    \label{fig:ex-sst2}
\end{figure}

\begin{figure}
    \centering
    \includegraphics[width=0.23\textwidth]{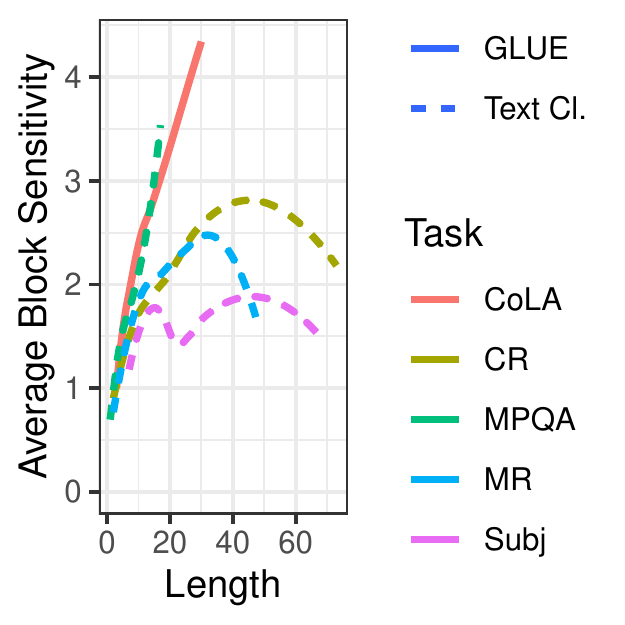}
    \includegraphics[width=0.23\textwidth]{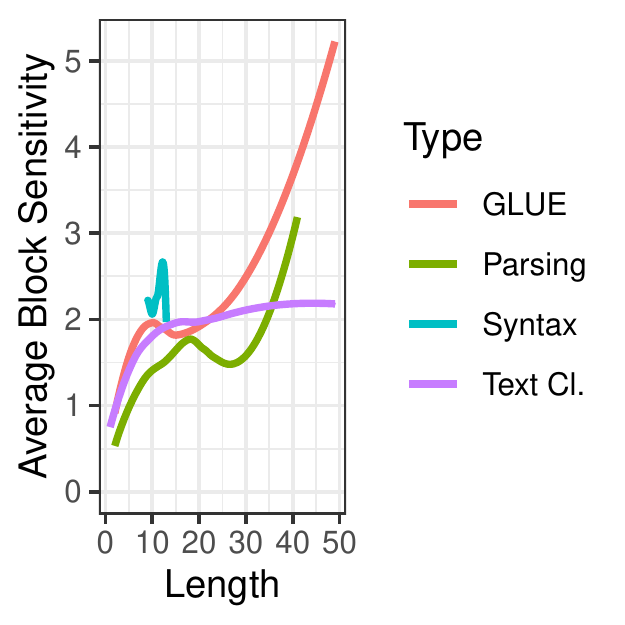}
	\caption{Per-example block sensitivity as a function of sentence length. \textit{Left:} Comparing text classification tasks with CoLA, a single-span GLUE task. \textit{Right:} Block sensitivity across task groups.}
    \label{fig:sens-by-length}
\end{figure}

\begin{figure}
    \centering
	\footnotesize{XLNet}

    \includegraphics[width=0.48\textwidth]{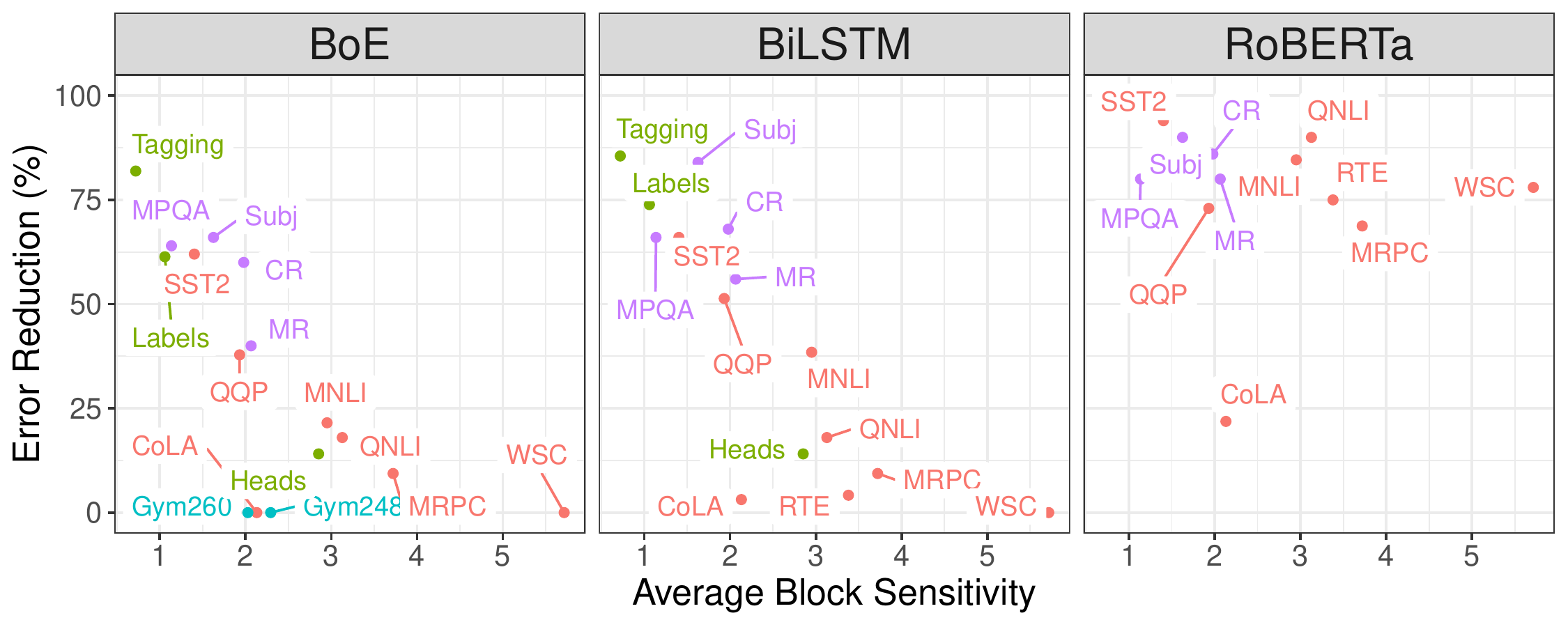}

	\footnotesize{u-PMLM}

    \includegraphics[width=0.48\textwidth]{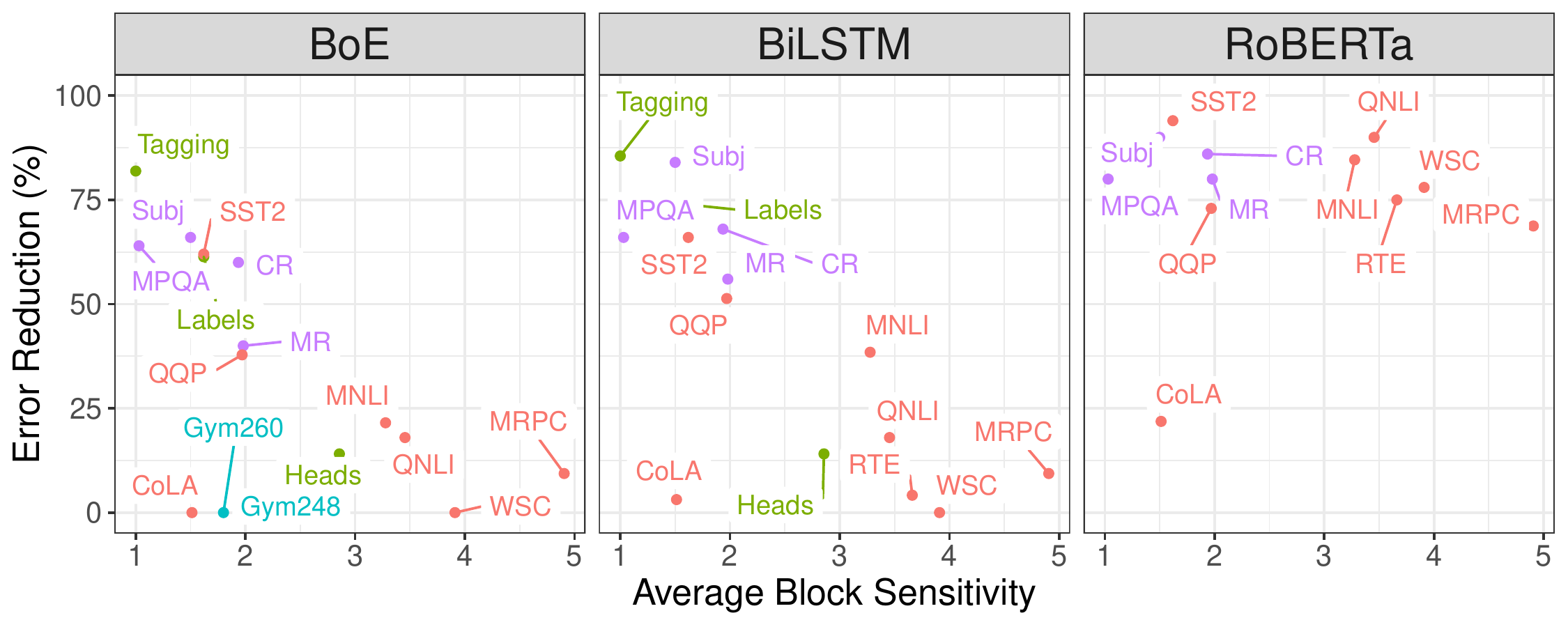}

    \includegraphics[width=0.48\textwidth]{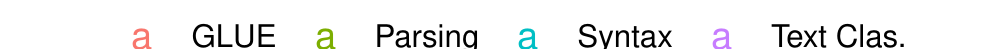}

	\caption{Sensitivity and simple models: average block sensitivity as estimated using XLNet (top) and u-PMLM (bottom) against error reduction (in \% of previously misclassified examples) of a Bag-of-Embeddings (BoE) model, a vanilla BiLSTM, and RoBERTa against the majority class baseline on the dev set.}
    \label{fig:sensitivity-cbow-acc}
\end{figure}

\begin{figure*}
\textbf{Premise:}

{\footnotesize
\setlength{\tabcolsep}{0.1em}
            \begin{tabular}{lllllll}
     & \textcolor{blue}{Steve Jobs} & was attacked by Sculley and & \textcolor{blue}{other Apple} & executives [...] and resigned from the company a few weeks later.  \\
    1. & \textcolor{red}{Chris Cook} & was attacked by Sculley and & other Apple & executives [...] and resigned from the company a few weeks later. \\
    2. & Steve Jobs & was attacked by Sculley and & \textcolor{red}{the other} & executives [...] and resigned from the company a few weeks later. \\
        \end{tabular}
        }

    \textbf{Hypothesis:}

    {\footnotesize
    \setlength{\tabcolsep}{0.1em}
        \begin{tabular}{lllllll}
    & \textcolor{blue}{Steve Jobs} & \textcolor{blue}{worked for} &  \textcolor{blue}{Apple}. \\
    3. & \textcolor{red}{Jobs  later} & worked for & Apple \\
     4. &Steve Jobs & \textcolor{red}{returned to} & Apple \\
     5. & Steve Jobs & worked for & \textcolor{red}{Google} \\
    \end{tabular}
    }
     
	\caption{An example from RTE, consisting of a premise and a hypothesis.
    In this example, the premise entails the hypothesis.
We show sensitive subsets $P_i$ identified by the models; for each of them, we show one of those completions created by XLNet that flip the label predicted by RoBERTa from \textsc{Entailment} to \textsc{NoEntailment}.
    In this example, five highly sensitive subsequences (two in the premise and three in the hypothesis) were identified.
    }
    \label{fig:ex-rte}
\end{figure*}

\begin{figure}
    \includegraphics[width=0.18\textwidth]{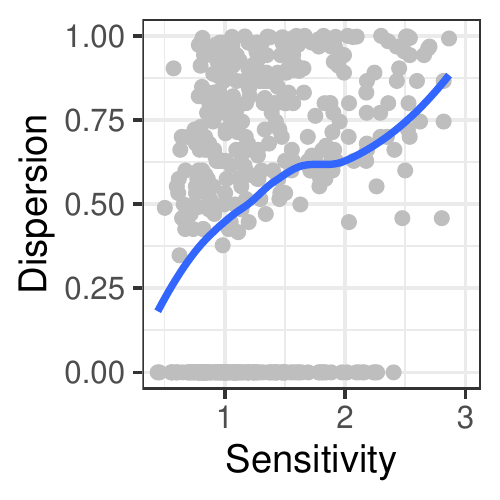}
    \includegraphics[width=0.28\textwidth]{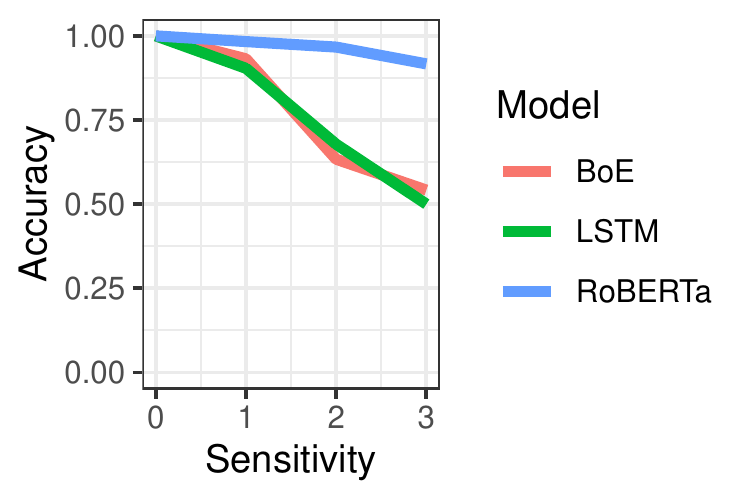}
	\caption{\textit{Left:} Block sensitivity and dispersion (see text) of sentiment labels of constituents. \textit{Right:} Accuracy as a function of sensitivity in sentiment analysis.}\label{fig:sst-sens}
\end{figure}

\paragraph{Average Block Sensitivity and Simple Models}

Based on Section~\ref{sec:bounds-models}, we hypothesized that tasks with low sensitivity correspond to those for which bag-of-words models can meaningfully outperform the majority class baseline, and those on which vanilla LSTM models do best.
In Figure~\ref{fig:sensitivity-cbow-acc}, we plot average block sensitivity against error reduction (in \% of previously misclassified examples) of a bag-of-embeddings (BoE) model\footnote{This model averages GloVE \citep{pennington2014glove} embeddings and applies a one-layer MLP to derive a prediction. This model is called CBOW in \citet{wang2019glue}; however, we apply BoE to concatenated spans in the case of multi-span tasks, in line with the definition of sensitivity.}, a vanilla BiLSTM\footnote{The syntax tasks have no training sets and we thus do not report BiLSTM results; we deduced necessarily at-chance performance for BoE from the design of the task. We excluded STS-B because it cannot be evaluated with accuracy.}, and RoBERTa against the majority class baseline, on the development sets.
BoE instantiates the model described in Proposition~\ref{prop:sensitivity-linear-model} with $k=1$; thus, we expect the top right of this graph to be empty for BoE: There can be no high-sensitivity task on which the BoE model provides strong quantitative performance.
For both BoE and the vanilla BiLSTM, average sensitivity was negatively associated with error reduction (XLNet: $R = -0.71$, $p=0.001$ for BoE; $R=-0.82$, $p=0.0002$ for BiLSTM. u-PMLM: $R=-0.66$, $p=0.005$ for BoE; $R=-0.76$, $p=0.002$ for BiLSTM), while no association was observed for RoBERTa (XLNet: $R=-0.05$, $p=0.87$; u-PMLM: $R=-0.07$, $p=0.84$).
We compared sensitivity as a predictor with label entropy, which showed little association with error reduction of either BoE or the vanilla BiLSTM (both $p > 0.1$).

\paragraph{Which inputs have high sensitivity?}
We used the Stanford Sentiment Treebank~\citep[SST-2,][]{socher2013recursive} to investigate which inputs have high sensitivity in sentiment classification.
We extracted the 445 dev inputs for which we had estimated sensitivity (determined by compute availability).
The dataset contains syntactic parses, with human sentiment annotation for each constituent.
We hypothesized that inputs have high sensitivity when different constituents have different sentiment.
We focus on estimates from XLNet for simplicity; results from u-PMLM are qualitatively identical.
We measured the dispersion of sentiment labels over constituents by enumerating positive ($+1$) and negative ($-1$) labels of all constituents, and computing the standard deviation of this resulting distribution; this is $1$ if as many constituents have positive sentiment as there are constituents with negative sentiment.
Figure~\ref{fig:sst-sens} (left) shows this dispersion measure as a function of sensitivity.
High-sensitivity examples have higher dispersion.
In a linear regression with dispersion and sentence length as predictors of sensitivity, dispersion was highly significant ($\beta = 0.53$, $p < 1.95\cdot 10^{-10}$), while length was not ($\beta = -0.00$, $p=0.49$).
This is illustrated by the examples in Figure~\ref{fig:ex-sst2} discussed above, where dispersion correlates with sensitivity:
The first example has low block sensitivity (0.93) and low label dispersion (0.0); the sentence is labeled positive and no constituent is labeled negative.
The second example has higher block sensitivity (1.88) and very high label dispersion (0.94): while the sentence is labeled positive, three constituents are labeled positive and five negative.

Second, we hypothesized that a BoE classifier and a vanilla LSTM perform better on low-sensitivity examples, whereas RoBERTa should provide better performance also on higher-sensitivity examples.
This is confirmed by Figure~\ref{fig:sst-sens} (right), where we show the accuracy of BoE, BiLSTM, and RoBERTa as a function of sensitivity.
In a logistic regression with sensitivity and sentence length as predictors of BoE accuracy, sensitivity was again highly significant ($\beta = -1.22$, $p = 4.1\cdot 10^{-10}$). 
Findings were similar for the BiLSTM ($\beta=-1.16$, $p=1.41\cdot 10^{-9}$). 
When predicting the accuracy of RoBERTa, there was still a measurable effect of sensitivity ($\beta=-1.37$, $p = 1.6\cdot 10^{-5}$), but overall Figure~\ref{fig:sst-sens} shows that RoBERTa provides more accurate predictions on higher-sensitivity input.
Sentence length was not a significant predictor for accuracy of any of the three models (all $p>0.05$).

If we choose $s(f,x)$ instead of $bs(f,x)$, i.e., restricting to singletons $P$, there is still a significant effect of $s(f,x)$ on BoE accuracy ($\beta=-1.24$, $p=1.1\cdot 10^{-6}$), but with inferior model fit compared to $bs(f,x)$ ($\Delta$Deviance $=20.0$), confirming block sensitivity as the more appropriate difficulty measure for simple NLP models.

\paragraph{Role of Task Model}
We have estimated sensitivity of GLUE and text classification tasks using a large pretrained transformer model (RoBERTa).
What would happen if we used a model outside of the family of massive pretrained contextual embeddings?
To answer this, we estimated $bs(f,x)$ on SST-2 and RTE using the vanilla BiLSTM to represent $f$.
On SST-2, sensitivity estimated with the BiLSTM's correlated with sensitivty estimated with RoBERTa on those inputs where the BiLSTM provides correct predictions ($R=0.36$, $p=2 \cdot 10^{-11}$), but not on those (typically higher-sensitivity ones) where its predictions are incorrect ($R=0.15$, $p=0.21$); a linear regression confirmed that RoBERTa's sensitivity was more predictive of the BiLSTM's sensitivity in those cases that the LSTM labeled correctly ($\beta=0.2$, $p=0.004$).
On RTE (where the BiLSTM's accuracy is at chance), the BiLSTM's sensitivity was at a constant low value ($\approx 0.5$) for all inputs.
This illustrates that automatic estimation of sensitivity requires a strong model that is able to achieve the sensitivity levels required by a task.

\paragraph{Role of Lower Bound Approximation}
We evaluated the role of the lower bound approximation on 20 inputs from SST-2 of between 8 and 11 words each -- long enough to make the approximation inexact but still allowing consideration of all $2^n$ subsets. 
We compared estimates of $bs(f,x)$ based on the approximation ($\leq 216$ subsets) and the full power set ($\leq 2^{11} = 2048$ subsets).
On average, the approximation decreased estimates of $b(f,x)$ from 1.59 to 1.35.
However, the two estimates were almost perfectly correlated ($R=0.95$, $p<10^{-10}$).
Even when restricting to singletons $P$ (up to 11 subsets), the correlation remained high ($R=0.81$, $p<0.0001$). 
Thus, while the approximation may underestimate the numerical values of $bs(f,x)$, it preserves the relative sensitivities of different inputs.

\section{Human Validation}\label{sec:human}

In Section~\ref{sec:measuring}, we estimated the sensitivity of NLP tasks by plugging a model $f'$ of the task $f$ into equation~\ref{eq:block-sens}.
This methodology requires that the model $f'$ provides good labels on the samples from $x^{\oplus P}$ obtained using the language models.
As the language models only approximate the input distribution, their samples could fall outside of the data distribution on which $f'$ approximates the true task $f$ at high accuracy.
If this were the case, high estimated sensitivity on tasks such as RTE might reflect brittleness of large models rather than true high sensitivity.
Here, we show that this is not the case:
Reasoning tasks like RTE have higher sensitivity than text classification tasks like SST-2, even when using human labels.

\subsection{Experiment 1: Validating Oracle Model}

For 60 items from SST2 and 30 items from RTE each, we collected the subsets $P_1, \dots, P_k$ achieving the maximum in~(\ref{eq:block-sens}), with 6 samples from XLNet for each subset (we collected fewer items from RTE because they typically have more sensitive subsets $P_i$, making annotation more expensive). 
We then recruited naive participants who labeled these samples; each sample was labeled by two or three annotators.
In addition to the appropriate labels (``positive'' and ``negative'' for SST-2, ``entails'' and ``does not entail'' for RTE), participants were also provided with a ``makes no sense'' option.
We repeated the study for SST2 both with and without finetuning. 

The rate of ``makes no sense'' responses on SST-2 was 18\% without finetuning and 11\% with finetuning; it was 12\% on RTE.
The agreement between RoBERTa and the modal human label was 80\% (without finetuning) and 85\% (with finetuning) on SST-2, and 72\% on RTE; compared to 87\%, 92\%, and 79\%, respectively average agreement between a single annotator and the modal label.
Interannotator agreement is below the human accuracies reported by \citet{nangia2019human}; we note that the creators of RTE specifically excluded items where human annotators did not agree \citep{dagan2009recognizing} and that SST-2 excludes reviews labeled as neutral~\citep{socher2013recursive}; we thus expect lower agreement on other strings from the same domain.

The key question is whether these levels of agreement guarantee consistent sensitivity estimates.
Figure~\ref{fig:expt1} (left) compares block sensitivity estimated using RoBERTa with values obtained by plugging in average human labels for the function $f(\cdot)$.
On both SST-2 and RTE, the values are strongly correlated (SST-2 with and without finetuning both $R=0.85$; RTE: $R=0.91$; all $p < 2.2\cdot 10^{-16}$). 
On RTE, human estimates are numerically lower than automatic estimates, but the difference in average sensitivity between SST-2 and RTE was strongly replicated by the human estimates ($\beta=1.3$, $p=1.3 \cdot 10^{-14}$ in a linear regression).
These results indicate that a strong model of a task leads to results similar to a human oracle. 
In particular, the qualitative difference in sensitivity between SST2 and RTE is replicated when using human labels.

\subsection{Experiment 2: Manual Approximation}
Experiment 1 showed that human and model labels yield similar results in estimating sensitivity.
However, we still relied on the subsets $P_i$ generated by the models.
Here, we show that sensitivity, both on the level of individual inputs and on the level of tasks, relates to human intuitions about the number of disjoint subsequences that can be changed to flip the label, which can be easily estimated without any model.

\definecolor{red}{HTML}{F8766D}
\definecolor{blue}{HTML}{00BFC4}

\definecolor{blue1}{HTML}{00BA38}
\definecolor{blue2}{HTML}{619CFF}

\begin{figure}
    \centering
\includegraphics[width=0.23\textwidth]{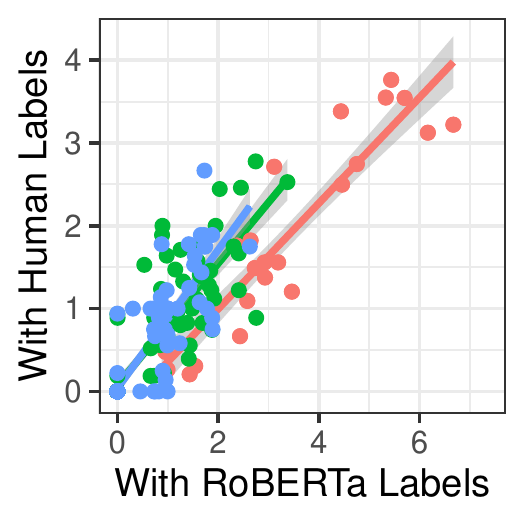}
\includegraphics[width=0.23\textwidth]{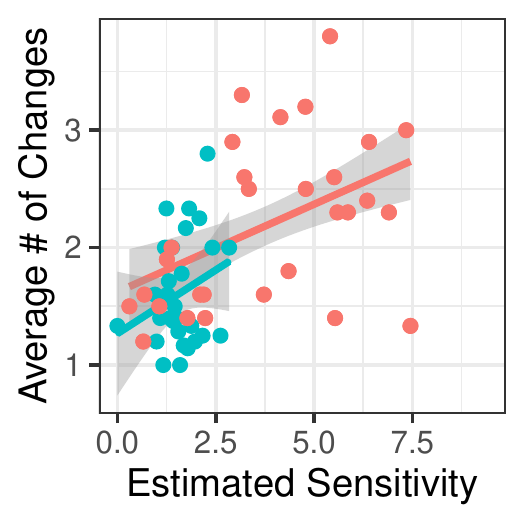}

\begin{center}
\begin{tabular}{llll}
	{\Huge{\textbf{\textcolor{red}{-}}}} \footnotesize{RTE}&
	{\Huge\textbf{\textcolor{blue2}{-}}} \footnotesize{SST-2 (raw)}\\
	\multicolumn{2}{c}{{\Huge{\textbf{\textcolor{blue1}{-}}}} \footnotesize{SST-2 (tuned)}} &
\end{tabular}
\begin{tabular}{llll}
	{\Huge\textbf{\textcolor{red}{-}}} \footnotesize{RTE}&
	{\Huge\textbf{\textcolor{blue}{-}}} \footnotesize{SST-2}\\
\end{tabular}
\end{center}

	\caption{Results of Experiments 1 and 2: \textit{Left:} Sensitivity on SST-2 and RTE calculated using RoBERTa's labels (x-axis) and using human labels (y-axis). On both tasks, both versions are highly correlated ($R > 0.8$ in both tasks). \textit{Right:} Results of Experiment 2: Average number of disjoint subsets on which participants change inputs to flip the label, as a function of estimated sensitivity on SST-2 and RTE.}
    \label{fig:expt1}\label{fig:expt2}
\end{figure}

We asked 30 naive subjects to find disjoint subsets in inputs from SST-2 and RTE such that changing the words in any one of them would flip the label.
Each subject worked on 30 items from one of the tasks.
Subjects rewrote sentences by clicking on words they wanted to replace and entering text replacing those.
After submitting a rewrite, subjects had the option of identifying another subset disjoint from the previously selected words.
Subjects changed at least one subset for every input, and were provided a bonus for every additional subset, incentivizing them to find as many disjoint subsequences as possible.
For both SST-2 and RTE, participants were shown an example with instructions guiding them to change three disjoint subsequences.
For RTE, we only allowed subjects to modify the premise, as similar changes are often possible in premise and hypothesis.

We interpreted the number of disjoint subsets found by participants as a proxy for block sensitivity.
This quantity is different from block sensitivity (\ref{eq:block-sens}), as it does not weight the relative probabilities of all possible changes, and we thus do not expect the same numerical values for both quantities.
An exact human estimate of block sensitivity would rely on asking humans both to create multiple samples from $x^{\oplus P}$ for different subsets $P$ and to then label these, infeasible given the large number of possible subsets of each input.
In contrast, the task described here only requires annotation from a few annotators for every input.

Figure~\ref{fig:expt2} (right) shows the average number of changes made on each input, as a function of the sensitivity estimated by XLNet+RoBERTa.
We conducted a mixed-effects Poisson regression of the number of changes made on the inputs, with random effects for items and subjects.
Sensitivity predicted the number of changes ($\beta = 0.061$, $SE=0.02$, $p=0.0023$), and there were overall more changes for RTE than for SST2 ($\beta=0.39$, $SE=0.097$, $p = 6\cdot 10^{-5}$).
Input length was not predictive ($\beta=-0.0015$, $SE=0.002$, $p=0.32$).
This result shows that a fully manual annotation task can approximate differences in sensitivity both between inputs (effect of sensitivity) and between tasks (effects of the contrast between RTE and SST-2).

\section{Discussion}
\label{sec:discussion}

We have proposed sensitivity as a theoretical framework for studying the complexity of sequence classification tasks, arguing that it captures complexity both across several machine learning architectures (Section~\ref{sec:bounds-models}) and across NLP tasks (Section~\ref{sec:measuring}).

Prior work has studied the ability of  RNNs and transformers to represent and learn languages in different classes of the Chomsky hierarchy \citep[e.g.,][]{merrill2019sequential}.
Sensitivity is orthogonal to the Chomsky hierarchy:
the maximally-sensitive function $f_{\text{Parity}}$ has a two-state finite automaton, but there are also low-sensitivity functions that are not even computable.
Sensitivity is also distinct from Kolmogorov complexity and similar description length measures~\citep{li1993an}: $f_\text{Parity}$ has high sensitivity but very low description length.
Whereas Kolmogorov complexity is uncomputable and can only be approximated asymptotically, sensitivity can be calculated for individual inputs, enabling us to explicitly evaluate it as a predictor of difficulty on NLP tasks.

\paragraph{Implications for NLP Practice}

Our results in Section~\ref{sec:measuring} suggest that pretrained contextualized embeddings have been so successful in NLP because they make it possible to learn high-sensitivity functions with modest amounts of task-specific training data.
We conjecture that, through large-scale pretraining, models implicitly learn high-sensitivity operations that are generally useful for language understanding. Finetuning such models for classification tasks~\citep{howard2018universal, peters2018deep, devlin2019bert} amounts to composing a high-sensitivity model with a low-sensitivity classifier.
Some classical techniques can also be interpreted in this light, such as aligning parse trees -- a potentially high-sensitivity computation -- and extracting features from these alignments that then are fed into an SVM -- a low-sensitivity classifier -- as an approach to tasks like RTE~\citep{dagan2009recognizing}.

\paragraph{Decision Boundaries in NLP}
The decision boundaries of NLP models are commonly studied to understand their linguistic knowledge \citep[e.g.][]{linzen2016assessing, marvin2018targeted, futrell2019neural, jeretic2020are}.
\citet{kaushik2020learning} and \citet{gardner2020evaluating} propose to improve NLP models and their evaluation by specifically considering input pairs that differ in some part and in their (true) label.
\citet{datta2020geometry} propose to quantify the difficulty of an input by the largest eigenvalue of the Fisher information matrix of a task model, finding that it predicts how sensitive classifiers are to word substitutions.

Sensitivity is different from widely studied phenomena of adversarial brittleness~\citep{szegedy2013intriguing, jia2017adversarial}:
The existence of adversarial examples typically means that natural examples have \emph{some} neighbors, possibly \emph{outside} of the input distribution, on which model output changes even though the true label does not.
In contrast, high sensitivity means that there are \emph{many} neighboring inputs \emph{within} the data distribution on which the \emph{true label} changes.
Sensitivity may be related to the observation that models often rely on spurious statistical patterns, such as simple lexical correlates of the label observed in reading comprehension datasets~\citep[e.g.][]{kaushik2018how,gururangan2018annotation}; we expect that such artifacts decrease task sensitivity as they make the gold labels correlated with the output of simple lexical classifiers.
Similarly, if the premise alone is predictive of the label in an entailment task~\citep{poliak2018hypothesis}, changing the hypothesis while staying within the task distribution is less likely to flip the label, again decreasing sensitivity.

\paragraph{Inductive Biases in Neural Networks}
There is empirical and theoretical evidence that the generalization capabilities of neural networks are in part due to a bias towards ``simple'' functions, with different formal notions of simplicity~\citep[e.g.][]{franco2006generalization,de2018deep,valle-perez2019deep}.
A few studies explicitly propose notions similar to low sensitivity as describing simplicity \citep{franco2006generalization,de2018deep,novak2018sensitivity}.
Relatedly, empirical work shows that neural networks learn low frequencies in the Fourier spectrum of functions first \citep{rahaman2019on,xu2019training,cao2019towards}.
As low average sensitivity corresponds to concentration of Fourier spectrum on low frequencies~\citep[Prop. 3.2]{odonnell2014analysis}, this can be understood as a bias towards low sensitivity.
One aspect distinguishing our results here from these prior studies is that we measure sensitivity of realistic functions arising as NLP tasks and on distributions reflecting the nontrivial statistics of natural language. Measuring sensitivity or Fourier spectra on other machine learning tasks is an interesting problem for future research.

\section{Conclusion}
\label{sec:conclusion}
We proposed block sensitivity as a complexity measure for functions from sequences to labels, applying the measure to quantify the complexity of sequence classification tasks in NLP.
Block sensitivity generalizes well-understood complexity measures from the theory of Boolean functions to the setting of natural language.
We showed both theoretically and empirically that low sensitivity characterizes tasks on which simple models without massive pretraining provide reasonable performance, and that, in such tasks, more difficult inputs correspond to those with high sensitivity.
Our results show that pretrained contextual embeddings enable models to learn tasks with higher sensitivity, and suggest designing challenging tasks by maximizing sensitivity,

\section*{Acknowledgments}
We thank Judith Degen, Kawin Ethayarajh, Mike Frank, Noah Goodman, and the members of the Stanford NLP group for helpful discussion and feedback.
We also thank the anonymous TACL reviewers for their insightful feedback that helped improve the paper.
We are also grateful to Yi Liao for generous help with running u-PMLM.

\bibliography{literature}
\bibliographystyle{acl_natbib}

\end{document}